\pdfoutput=1
\documentclass{article}
\usepackage[accepted]{icml2023}
\usepackage{aux/amirhesam_macros}
\usepackage{aux/parthe_macros}
\usepackage{wrapfig}
\renewcommand{\thefootnote}{\fnsymbol{footnote}}
\newcommand{\EP}[1]{EigenPro\,#1}
\newcommand{\FALKON}{\textsf{\textsc{\small Falkon}}}

\usepackage{epstopdf}
\newcommand{\GPytroch}{\textsf{\textsc{\small Gpytorch}}}
\newcommand{\NYTRO}{\textsf{\textsc{\small Nytro}}}
\renewcommand{\Hilbert}{{\bm{\mathcal{H}}}}
\newcommand{\Xspan}{\bm{\mathcal{X}}}
\newcommand{\Zspan}{\bm{\mathcal{Z}}}

\def\cite{\citet}

\begin{document}
\icmltitlerunning{Toward large kernel models}

\twocolumn[
\icmltitle{Toward Large Kernel Models}

\begin{icmlauthorlist}
\icmlauthor{Amirhesam Abedsoltan}{cse}
\icmlauthor{Mikhail Belkin}{hdsi,cse}
\icmlauthor{Parthe Pandit}{hdsi}
\end{icmlauthorlist}

\icmlaffiliation{cse}{Department of Computer Science and Engineering, and}
\icmlaffiliation{hdsi}{Halicioglu Data Science Institute, UC San Diego, USA}

\icmlcorrespondingauthor{}{aabedsoltan, parthepandit@ucsd.edu}

\icmlkeywords{Kernel machines}

\vskip 0.3in
]
\printAffiliationsAndNotice{}

\newcommand\blfootnote[1]{%
  \begingroup
  \renewcommand\thefootnote{}\footnote{#1}%
  \addtocounter{footnote}{-1}%
  \endgroup
}

\begin{abstract}

Recent studies indicate that kernel machines can often perform similarly  or better than deep neural networks (DNNs) on small datasets. The interest in kernel machines has been additionally bolstered by the discovery of their equivalence to wide neural networks in certain regimes. 
However, a key feature of DNNs is their ability to scale the model size and training data size independently, whereas in traditional kernel machines model size is tied to data size. Because of this coupling, scaling kernel machines to large data has been computationally challenging. 
In this paper, we provide a way forward  for constructing large-scale \textit{general kernel models}, which are a generalization of kernel machines that decouples the model and data, allowing training on large datasets. 
Specifically, we introduce EigenPro 3.0, an algorithm based on  projected dual preconditioned SGD and show scaling to model and data sizes which have not been possible with existing kernel methods. We provide a PyTorch based implementation which can take advantage of multiple GPUs.
\end{abstract}

\vspace{-0.5cm}
\section{Introduction}

Deep neural networks (DNNs) have become the gold standard for many large-scale machine learning tasks. Two key factors that contribute to the success of DNNs are the large model sizes and the large number of training samples. Quoting from~\citep{kaplan2020scaling} {``\it performance depends most strongly on scale, which consists of three factors: the number of model parameters $N$ (excluding embeddings), the size of the dataset $D$, and the amount of compute $C$ used for training. Within reasonable limits,
performance depends very weakly on other architectural hyperparameters such as depth vs. width}''.
Major community  effort and great amount of resources have been invested in scaling models and data size, as well as in understanding the relationship between the number of model parameters, compute, data size, and  performance. 
Many current architectures have hundreds of billions of parameters and are trained on large datasets with nearly a trillion data points (e.g., Table~1 in~\citep{hoffmann2022training}). Scaling both model size and the number of training samples are seen as crucial for optimal performance.

Recently, there has been a surge in research on the equivalence of special cases of DNNs and kernel machines. For instance, the Neural Tangent Kernel (NTK) has been used to understand the behavior of fully-connected DNNs in the infinite width limit by using a fixed kernel \citep{jacot2018neural}. A rather general version of that phenomenon was shown in \citep{zhu2022transition}. Similarly, the Convolutional Neural Tangent Kernel (CNTK) \citep{li2019enhanced} is the NTK for convolutional neural networks, and has been shown to achieve accuracy comparable to AlexNet \citep{alexnet2012} on the CIFAR10 dataset.\blfootnote{A Python package is available at \href{https://github.com/EigenPro/EigenPro3}{github.com/EigenPro3}}

These developments have sparked interest in the potential of kernel machines as an alternative for DNNs. Kernel machines are relatively well-understood theoretically, are stable, somewhat interpretable, and have been shown to perform similarly to DNNs on small datasets \citep{arora2020harnessing,lee2020finite, radhakrishnan2022simple}, particularly on tabular data~\citep{HGammaKernel, radhakrishnan2022feature}. However, in order for kernels to be a viable alternative to DNNs, it is necessary to develop methods to scale kernel machines to large datasets.

\textbf{The problem of scaling.} 
Similarly to DNN, to achieve optimal performance of kernel models, it is not sufficient to just increase the size of the training set for a fixed model size, but the model size must scale as well. Fig.~\ref{small_model} illustrates this property on a small-scale example (see Appendix \ref{sec:fig1_app} for the details). The figure demonstrates that the best performance cannot be achieved solely by increasing the dataset size. Once the model reaches its capacity, adding more data leads to marginal, if any, performance improvements. On the other hand, we see that the saturation point for each model is not achieved until the number of samples significantly exceeds the model size.  This illustration highlights the need for algorithms that can independently scale dataset size and model size for optimal performance.

\begin{figure}[t]
  \centering
  \vspace{-10pt} 
  \includegraphics[width=1.0\columnwidth]{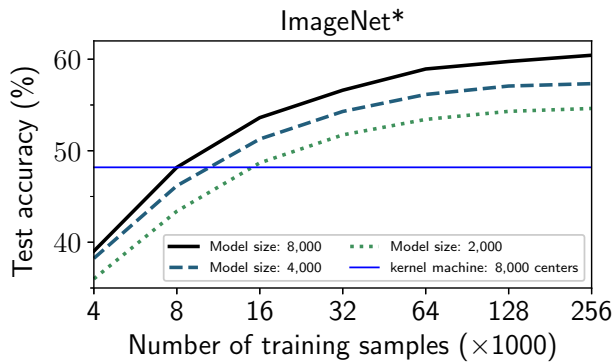}
 \vspace{-0.8cm} 
  \caption{Increasing number of training samples for a fixed model size is helpful but insufficient for optimal performance. The model size must scale as well. See Appendix \ref{sec:fig1_app} for details.} 
  \vspace{-0.6cm}\label{small_model}
\end{figure}

\subsection{Main contributions} 

We introduce \EP{3.0} for learning general kernel models that can handle  \textbf{large model sizes}.
In our numerical experiments, we train models with up to $1$ million centers on $5$ million samples. To the best of our knowledge, this was not achievable with any other existing method. Our work provides a path forward to scaling both the model size and size of the training dataset, independently. 

\subsection{Prior work} \label{subsec:priorworks}
A naive approach for training kernel machines is to directly solve the equivalent kernel matrix inversion problem. In general, the computational complexity of solving the kernel matrix inversion problem is $O(n^3)$, where $n$ is the number of training samples. Thus, computational cost grows rapidly with the size of the dataset, making it computationally intractable for datasets with more than $\sim10^5$ data points.

A number of methods have been proposed to tackle this issue by utilizing various iterative algorithms and approximations. We provide a brief overview of the key ideas below.

\textbf{Gradient Descent(GD) based algorithms:} Gradient descent (GD) based methods, such as Pegasos~\citep{shalev2007pegasos}, have a more manageable computational complexity of $O(n^2)$ and can be used in a stochastic setting, allowing for more efficient implementation. Preconditioned stochastic gradient descent based method, EigenPro~\citep{ma2017diving} was introduced to accelerate convergence of Pegasos. EigenPro is an iterative algorithm for kernel machines that uses a preconditioned Richardson iteration \citep{richardson1911ix}. Its performance was further improved in EigenPro 2.0~\citep{ma2019kernel} by reducing the computational and memory costs for the preconditioner through the use of a Nyström extension~\citep{williams2000using} and introducing full hardware (GPU) utilization by adaptively auto-tuning the learning rate.

However, Kernel machines are limited in scalability due to the coupling between the model and the training set. With modern hardware, these methods can handle just above one million training data points. A more in-depth discussion about EigenPro can be found in Section~\ref{subsec:Background on EigenPro}.


\textbf{Large scale Nystr\"om-approximate models:} Nystr\"om methods have been a popular  strategy for applying kernel machines at scale starting with~\citep{williams2000using}. We refer the reader to \Cref{sec:prelim} for a detailed description of the Nystr\"om approximation.
Methods such as {\NYTRO}~\citep{camoriano2016nytro} and {\FALKON}~\citep{rudi2017falkon} use Nystr\"om approximation (NA) in combination with other techniques to improve performance. Specifically, {\NYTRO} combines NA with gradient descent to improve the condition number, while {\FALKON} uses NA in combination with the Conjugate Gradient method. While these methods allow for very large training sets, they are limited in terms of the model size due to memory limitations.  For example, the largest models considered~\citep{meanti2020kernel} have about $100,\!000$ centers. With $340$GB RAM available to us on the Expanse cluster~\citep{XSEDE}, we were able to run \FALKON~ with at $256,\!000$ model size (Figure~\ref{fig:falkonvsep2vsep3}). However, scaling to $512,\!000$ model size, already requires over $1$TB RAM which goes beyond specifications of most current  high end servers. 

\textbf{Sparse Gaussian Process:} Methods from the literature on Gaussian Processes, e.g. \citep{titsias2009variational}, use  so-called \textit{inducing points} to control the model complexity. While several follow-ups such as \citep{wilson2015kernel} and \GPytroch~ \citep{gardner2018product}, and \textsf{\small GPFlow}~\citep{GPflow2017} have been applied in practice, they require quadratic memory in terms of the number of the inducing points, thus preventing scaling to large models. 
A closely related concept is that of \textit{kernel herding} \citep{chen2012super}. The focus of these methods is largely on selecting ``good'' model centers, rather than scaling-up the training procedure. Our work is complementary to that line of research.

\textbf{Random Fourier Features (RFFs):} 
RFF is a popular method first introduced in \citep{rahimi2007random} to approximate kernel machines using so-called ``Random Fourier Features''. However, it is generally believed  that Nystr\"om methods outperform RFF ~\citep{yang2012nystrom}.

\section{Preliminaries and Background}
\label{sec:prelim}

\textbf{Kernel Machines:} Kernel machines, \citep{scholkopf2002learning} are non-parametric predictive models. Given training data $(X,\y)=\curly{\x_i\in \Real^d,y_i\in \Real}_{i=1}^n$  a kernel machine is a model of the form
\begin{align}\label{eq:representer_thm}
    f(\x) = \sum_{i=1}^n\alpha_i K(\x,\x_i).
\end{align}
Here, $K:\mathbb{R}^d\times \mathbb{R}^d\rightarrow\mathbb{R}$ is a positive semi-definite symmetric kernel function \citep{aronszajn1950theory}. According to the representer theorem \citep{kimeldorf1970correspondence}, the unique solution to the infinite-dimensional optimization problem
\begin{align}\label{eq:kernel_regression}
    &\arg\min{f\in\Hilbert}\, \sum_{i=1}^n (f(\x_i)-y_i)^2+\lambda\norm{f}_\Hilbert^2
\end{align}
has the form given in Eq.~\ref{eq:representer_thm}.

Here $\Hilbert$ is the (unique) reproducing kernel Hilbert space (RKHS) corresponding to $K$.

It can be seen that $\alphavec = (\alpha_1,\ldots,\alpha_n)$ in \cref{eq:representer_thm} is the unique solution to the linear system,
\begin{align}\label{eq:linear_system}
(K(X,X)+\lambda I_n)\alphavec=\y.
\end{align}

\textbf{General kernel models} are models of the form,
\begin{align*}
f(\x) = \sum_{i=1}^p\alpha_i K(\x,\z_i),
\end{align*}
where $Z=\{z_i\in\mathbb{R}^d\}_{i=1}^p$ is the set of \textit{centers}, which is not necessary the same as the training set. We will refer to $p$ as the \textit{model size}. 
Note that kernel machines are a special type of general kernel models with $Z = X$. Our goal will be to solve \cref{eq:kernel_regression} with this new constraint on $f$.

Note that when $Z\subset X$, is a random subset of the training data, the resulting model is called a Nystr\"om approximate (NA) model \citep{williams2000using}.

In contrast to kernel machines, general kernel models, such as classical RBF networks~\citep{poggio1990networks},  allow for the separation of the model and training set.  
General kernel models also provide explicit control over model capacity by allowing the user to choose $p$ separately from the training data size. This makes them a valuable tool for many applications, particularly when dealing with large datasets.

\textbf{Notation:} In what follows, functions are lowercase letters $a$, sets are uppercase letters $A$, vectors are lowercase bold letters $\a$, matrices are uppercase bold letters $\A$, operators are calligraphic letters $\mathcal{A},$ spaces and sub-spaces are boldface calligraphic letters $\bm{\mathcal{A}}.$

\textbf{Evaluations and kernel matrices:} The vector of evaluations of a function $f$ over a set $X=\curly{\x_i}_{i=1}^n$ is denoted $f(X):=(f(\x_i))\in\Real^n$. For sets $X$ and $Z$, with $|X|=n$ and $|Z|=p$, we denote the kernel matrix $K(X,Z)\in\Real^{n\times p},$ while $K(Z,X)=K(X,Z)\T$.
Similarly, $K(\cdot,X)\in\Hilbert^{n}$ is a vector of functions, and we use $K(\cdot,X)\alphavec:=\sum_{i=1}^n K(\cdot,\x_i)\alpha_i\in\Hilbert,$ to denote their linear combination.
Finally, for an operator $\mc A,$ a function $a$, and a set $A=\{\a_i\}_{i=1}^k$, we denote the vector of evaluations of the output, 
\begin{align}
 \mc A\curly{a}(A)   := (b(\a_i))\in\Real^k\qquad \text{where}\quad b=\mc A\round{a}.
\end{align}

\begin{definition}[Top-$q$ eigensystem]\label{def:topq_eignesystem} Let $\lambda_1>\lambda_2>\ldots>\lambda_n,$ be the eigenvalues of a hermitian matrix $\A\in\Real^{n\times n}$, \ie, for unit-norm $\e_i$, we have $\A\e_i=\lambda_i\e_i$. Then we call the tuple $(\Lambda_q,\E_q,\lambda_{q+1})$ the top-$q$ eigensystem, where \begin{align}
    &\Lambda_q:={\rm diag}(\lambda_1,\lambda_2,\ldots,\lambda_q)\in\Real^{q\times q},\text{ and}\\
    &\E_q:=[\e_1,\e_2,\ldots,\e_q]\in\Real^{n\times q}.
\end{align}
\end{definition}

\begin{definition}[Fr\'echet derivative] Given a function $J:\Hilbert\to \Real$,
the Fr\'echet derivative of $J$ with respect to $f$ is a linear functional, denoted $\nabla_f J$, such that for $h\in\Hilbert$
\begin{align}
    \lim_{\norm{h}_\Hilbert\rightarrow 0}\frac{\abs{J(f+h)-J(f)-\nabla_f J(h)}}{\norm{h}_\Hilbert}=0.
\end{align}
\end{definition}
Since $\nabla_f J$ is a linear functional, it lies in the dual space $\Hilbert^*.$ Since $\Hilbert$ is a Hilbert space, it is self-dual, whereby $\Hilbert^*=\Hilbert.$
If $f$ is a general kernel model, and $L$ is the square loss for a given dataset $(X,\y)$, \ie, $L(f):=\frac12\sum_{i=1}^n(f(\x_i)-y_i)^2$ we can apply the chain rule, and using reproducing property of $\Hilbert$, and the fact that $\nabla_f\inner{f,g}_\Hilbert=g$, we get, that the Fr\'echet derivative of $L$, at $f=f_0$ is, 
\begin{align}\label{eq:frechet_square_loss}
    \nabla_f L(f_0)&=\sum_{i=1}^n(f_0(\x_i)-y_i)\nabla_{\!f} f(\x_i)\\
    &=K(\cdot,X)(f_0(X)-\y).
\end{align}

\textbf{Hessian operator: } The Hessian operator $\nabla^2_f L:\Hilbert\rightarrow\Hilbert$ for the square loss is given by,
\begin{subequations}\label{eq:hessian}
\begin{align}
    &\mc K:= \sum_{i=1}^n K(\cdot,\x_i)\otimes K(\cdot,\x_i),\\
    &\mc K\curly{f}(\z)\!=\!
    \sum_{i=1}^n K(\z,\x_i)f(\x_i)=K(\z,X)f(X).
\end{align}
\end{subequations}
Note that $\mc K$ is surjective on $\Xspan,$ and hence invertible when restricted to $\Xspan$.  
Note that when $\x_i\overset{\rm i.i.d.}\sim\mathbb{P}$, for some measure $\mathbb{P}$, the above summation, on rescaling by $\frac1n$, converges due to the strong law of large numbers as,
\begin{align}\label{eq:integral_operator}
    \lim_{n\rightarrow\infty}\frac{\mc K\curly{f}}n =\mc T_K\curly{f}:=\int K(\cdot,\x)f(\x)\dif\mathbb{P}(\x),
\end{align}
which is the integral operator associated with a kernel $K$.
The following lemma relates the spectra of $\mc K$ and $K(X,X)$.
\begin{proposition}[Nystr\"om extension]\label{prop:nystrom_extension}
For $1\leq i\leq n$, let $\lambda_i$ be an eigenvalue of $\mc K$, and $\psi_i$ its unit $\Hilbert$-norm eigenfunction, $\mc K\curly{\psi_i}=\lambda_i\psi_i$. Then $\lambda_i$ is also an eigenvalue of $K(X,X)$. Moreover if $\e_i,$ is a unit-norm eigenvector, $K(X,X)\e_i=\lambda_i\e_i$, we have,
\begin{align}
 \psi_i  = K(\cdot,X)\frac{\e_i}{\sqrt{\lambda_i}}=\sum_{j=1}^n K(\cdot, \x_j)\frac{e_{ij}}{\sqrt{\lambda_i}}.
\end{align}
\end{proposition}

We review \EP{2.0} which is a closely related algorithm for kernel regression, \ie, when $Z=X$.

\subsection{Background on EigenPro}\label{subsec:Background on EigenPro}
\EP{1.0}, proposed in~\citep{ma2017diving}, is an iterative solver for solving the linear system in equation \eqref{eq:linear_system} based on a preconditioned stochastic gradient descent in a Hilbert space,
\begin{align}\label{eq:GD}
    &f^{t+1}=f^{t} - \eta\cdot\mc P\curly{\nabla_f L(f^t)}.
\end{align}
Here $\mc P$ is a preconditioner.
Due to its iterative nature, EigenPro can handle $\lambda=0$ in equation \cref{eq:linear_system}, corresponding to the problem of kernel interpolation, since in that case, the learned model satisfies $f(\x_i)=y_i$ for all samples in the training-set. 

It can be shown that the following iteration in $\Real^n$ 
\begin{align}
    \alphavec^{t+1}=\alphavec^{t}-\eta(\I_n-\Q)(K(X,X)\alphavec^t-\y),
\end{align}
emulates \cref{eq:GD} in $\Hilbert,$ see \Cref{lem:EigenPro_Rn_equivalence} in the Appendix.
The above iteration is a preconditioned version of the Richardson iteration, \citep{richardson1911ix}, with well-known convergence properties.
Here, $\Q$ as a rank-$q$ symmetric matrix obtained from the top-$q$ eigensystem of $K(X, X),$ with $q\ll n$. Importantly $\Q$ commutes with $K(X,X).$

The preconditioner, $\mc P$ acts to flatten the spectrum of the Hessian $\mc K$. In $\Real^n$, the matrix $ \I_n-\Q$ has the same effect on $K(X, X).$ The largest stable learning rate is then $\frac{2}{\lambda_{q+1}}$ instead of $\frac{2}{\lambda_1}$. Hence a larger $q,$ allows faster training when $\mc P$ is chosen appropriately.

\EP{2.0} proposed in \citep{ma2019kernel}, applies a stochastic approximation for $\mc P$ based on the Nystr\"om extension. We apply \EP{2.0} to perform an inexact projection step in our algorithm. 
\vspace{-0.15cm}
\section{Problem Formulation}

In this work we aim to learn a general kernel model to minimize the square loss over a training set.
We will solve the following infinite dimensional convex constrained optimization problem in a scalable manner.
\begin{subequations}\label{eq:loss_min}
\begin{align}
    &\minimize{f}\, L(f)=\sum_{i=1}^n (f(\x_i)-y_i)^2,\\
    &\subjectto\quad {f\in\Zspan}:=\text{span}\!\round{\curly{K(\cdot,\z_j)}_{j=1}^p}.
\end{align}
\end{subequations}

The term "scalable" refers to both large sample size ($n$) and large model size ($p$). For instance Figure \ref{fig:bigmodels} shows an experiment with  $5$ million samples and the model size of $1$ million. Our algorithm to solve this problem, \EP{3.0}, is derived using a projected preconditioned gradient descent iteration.
\section{\texorpdfstring{{\sf EigenPro\,3.0}}{} derivation: Projected preconditioned gradient descent}
\label{sec:derivation}

In this section we derive \EP{3.0-Exact-Projection} (\Cref{alg:eigenpro30exact}), a precursor to \EP{3.0}, to learn general kernel models. This algorithm is based on a function space projected gradient method. However it does not scale well. In \Cref{sec:speedup} we make it scalable by applying stochastic approximations, which finally yields \EP{3.0} (\Cref{alg:eigenpro30}).

We will apply a function-space projected gradient method to solve this problem,
\begin{align}\label{eq:PGM}
    f^{t+1} = \text{proj}_{\Zspan}\round{f^t-\eta\mc P\curly{\nabla_fL(f^t)}},
\end{align}
where $\text{proj}_{\Zspan}\round{u} := \argmin{f\in\Zspan}\norm{u-f}_\Hilbert^2$, and $\nabla_fL(f^t)$ is the Fr\'echet derivative at $f^t$ as given in \cref{eq:frechet_square_loss}, $\mc P$ is a preconditioning operator given in \cref{eq:preconditioner}, $\eta$ is a learning rate. 
Note that the operator $\text{proj}_{\Zspan}:\Hilbert\rightarrow\Zspan$ projects functions from $\Hilbert$ onto the subspace $\Zspan$, ensuring feasibility of the optimization problem.
\begin{remark} 
Note that even though \cref{eq:PGM} is an iteration over functions which are infinite dimensional objects $\curly{f^t}_{t\geq0}$, we can represent this iteration in finite dimensions as $\curly{\alphavec^t}_{t\geq0},$ where $\alphavec_t\in\Real^p$.
To see this, observe that $f^t\in\Zspan$, whereby we express it as,
\begin{align}
f^t=K(\cdot,Z)\alphavec^t\in\Hilbert,\qquad\text{for an  $\alphavec^t\in\Real^p.$}
\end{align}\label{eq:finite_rep}
Furthermore, the evaluation of $f^t$ above at $X$, is
\begin{align}\label{eq:evaluation}
    f^t(X)=K(X,Z)\alphavec^t\in\Real^n.
\end{align}
\end{remark}

\subsection{Gradient} Due to equations \eqref{eq:frechet_square_loss} and \eqref{eq:evaluation} together, the gradient is given by the function, 
\begin{subequations}\label{eq:grad}
\begin{align}
    \nabla_{\!f}L(f^t) &= K(\cdot,X)(f^t(X)-\y)\\
    &= K(\cdot,X)(K(X,Z)\alphavec^t-\y)\in\Xspan\\
\Xspan&:=\text{span}\!\round{\curly{K(\cdot,\x_i)}_{i=1}^n}.
\end{align}
\end{subequations}
Observe that the gradient does not lie in $\Zspan$ and hence a step of gradient descent would leave $\Zspan,$ and the constraint is violated.
This necessitates a projection onto $\Zspan$. For finitely generated sub-spaces such as $\Zspan$, the projection operation involves solving a finite dimensional linear system.

\subsection{\texorpdfstring{$\Hilbert$}{H}-norm projection}
Functions in $\Zspan$ can be expressed as $K(\cdot,Z)\thetavec$. Hence we can rewrite the projection problem in \cref{eq:PGM} as a minimization in $\Real^p$, with $\thetavec$ as the unknowns. Observe that, 
\begin{align*}
    \argmin{f}\norm{f-u}_\Hilbert^2 =\argmin{f}
    \inner{f,f}_\Hilbert  -2\inner{f,u}_\Hilbert  
\end{align*} 
since $\norm{u}_\Hilbert^2$ does not affect the solution. Further, using $f=K(\cdot,Z)\thetavec$, we can show that
\begin{align}
\inner{f,f}_\Hilbert  -2\inner{f,u}_\Hilbert 
    = \thetavec\tran K(Z,Z)\thetavec - 2\thetavec\tran u(Z).
\end{align}
This yields a simple method to calculate the projection onto $\Zspan.$
\begin{subequations}\label{eq:proj_equiv}
\begin{align}
    \text{proj}_{\Zspan}\{u\} & 
    =\argmin{f\in\Zspan}\norm{f-u}_\Hilbert^2= K(\cdot,Z)\wh\thetavec\\
    &=K(\cdot,Z)K(Z,Z)\inv u(Z)\in\Zspan,
    \end{align}
\end{subequations}
where
\begin{align*}
    \wh\thetavec
    &= \argmin{\thetavec\in\Real^p}\thetavec\tran K(Z,Z)\thetavec - 2\thetavec\tran u(Z)= K(Z,Z)\inv u(Z).
\end{align*}
Notice that $\wh \thetavec$ above is linear in $u,$ and $f^t(Z)=K(Z,Z)\alphavec^t.$ Hence we have the following lemma.

\begin{proposition}
[Projection]\label{lem:projection} The projection step in \cref{eq:PGM} can be simplified as,
\begin{align}\label{eq:projection_update}
    f^{t+1}= f^t - \eta\,K(\cdot,Z)&K(Z,Z)\inv\times\nonumber\\
    &\round{\mc P\curly{\nabla_fL(f^t)}(Z)} \in\Zspan.
\end{align}
\end{proposition}

Hence, in order to perform the update, we only need to know $\mc P\curly{\nabla_fL(f^t)}(Z)$, \ie, the evaluation of the preconditioned Fr\'echet derivative at the model centers.
This can be evaluated efficiently as described below.

\begin{algorithm}[t]
\caption{\EP{3.0\,Exact-Projection}}\label{alg:eigenpro30exact}
\begin{algorithmic}[1]
\Require Data $(X,y)$, centers $Z,$ initialization $\alphavec^0,$ preconditioning level $q.$
\State $(\Lambda,\E,\lambda_{q+1})\gets$ top-$q$ eigensystem of $K(X,X)$
\State $\Q\gets\E(\I_q-\lambda_{q+1}\Lambda\inv)\E\tran\in\Real^{n\times n}$
\While{Stopping criterion not reached}
\State $\g \gets K(X,Z)\alphavec-\y$
\State $\h\gets K(Z,X)(\I_n-\Q)\g$
\State $\thetavec\gets K(Z,Z)\inv  \h$ 

\State $\alphavec \gets \alphavec - \eta\, \thetavec$
\EndWhile
\end{algorithmic}
\end{algorithm}

\subsection{Preconditioner agnostic to the model}
Just like with usual gradient descent, the largest stable learning rate is governed by the largest eigenvalue of the Hessian of the objective in \cref{eq:loss_min}, which is given by \cref{eq:hessian}. The preconditioner $\mc P$ in \cref{eq:PGM} acts to reduce the effect of a few large eigenvalues. We choose $\mc P$ given in   \cref{eq:preconditioner}, just like \citep{ma2017diving}.
\begin{align}
\mc P:=\mc{I} - \sum_{i=1}^q \round{1-\frac{\lambda_{q+1}}{\lambda_{q}}} \psi_i\otimes\psi_i\quad :\Hilbert\rightarrow\Hilbert.\label{eq:preconditioner}
\end{align}
Recall from \Cref{sec:prelim} that $\psi_i$ are eigenfunctions of the Hessian $\mc K$, characterized in \Cref{prop:nystrom_extension}. Note that this preconditioner is independent of $Z$.  
Since $\nabla_f L(f^t)\in\Xspan$,  we only need to understand $\mc P$ on $\Xspan.$ 
Let $(\Lambda_q,\E_q,\lambda_{q+1})$ be the top-$q$ eigensystem of $K(X,X)$, see  \Cref{def:topq_eignesystem}. Define the rank-$q$ matrix,
\begin{align}
    \Q:=\E_q(\I_q-\lambda_{q+1}\Lambda_{q}\inv)\E_q\tran\in\Real^{n\times n}.
\end{align}
The following lemma outlines the computation involved in preconditioning. 
\begin{proposition}[Preconditioner] The action of $\mc P$ from \cref{eq:preconditioner} on functions in $\Xspan$ is given by,
\begin{align}
    \mc P\curly{K(\cdot,X) \a} =  K(\cdot,X)(\I_n-\Q)\a,
\end{align}
for all $\a\in\Real^m.$
\end{proposition}
Since we know from \cref{eq:grad} that $\nabla_f L(f^t)=K(\cdot,X)(K(X,Z)\alphavec^t-\y)$, we have,
\begin{align*}
    \mc P\curly{\nabla_fL(f^t)}(Z)\!=\!K(Z,X)(\I_n\!-\!\Q)(K(X,Z)\alphavec^t\!-\!\y).
\end{align*}
The following lemma combines this with \Cref{lem:projection} to get the update equation for \Cref{alg:eigenpro30exact}.

\begin{lemma}[\Cref{alg:eigenpro30exact} iteration] The following iteration in $\Real^p$ emulates \cref{eq:PGM} in $\Hilbert$,
\begin{align}
    \alphavec^{t+1} 
    = \alphavec^t - \eta\, K(Z,Z)\inv &K(Z,X)(\I_n-\Q)\times\nonumber\\
    &(K(X,Z)\alphavec^t-\y).
\end{align}
\end{lemma}

\begin{table*}[t]
\centering
\begin{tabular}{|c||ccc|c|}
\hline
\multirow{2}{*}{Algorithm}  
& \multicolumn{3}{c||}{Compution}                 & \multirow{2}{*}{Memory} \\ \cline{2-4}
& \multicolumn{1}{c|}{Setup} 
& \multicolumn{2}{c||}{per iteration} 
&        
\\ \hline
\EP{3.0}                
& \multicolumn{1}{c|}{$s^2q$}      
& \multicolumn{2}{c||}{$p(m+s)+s(m+q)+T_{\rm ep2}$}              
& $s^2+sm+M_{\rm ep2}$ 
\\ \hline
\EP{3.0} ExactProjection 
& \multicolumn{1}{c|}{${nq^2+p^3}$}      
& \multicolumn{2}{c||}{${np+nq}$}              
&  $pn+n^2$      
\\ \hline
\FALKON                  
& \multicolumn{1}{c|}{$p^3$}      
& \multicolumn{2}{c||}{$np$}              
& $p^2$
\\ \hline
\end{tabular}

\caption{\label{tab:computational complexity}
   \textbf{Algorithm complexity.} Number of training samples $n$, number of model centers $p$, batch size $m$, Nystr\"om sub-sample size $s$, preconditioner level $q$.
    Here $T_{\rm ep2}$ is the time it takes to run \EP{2.0} for the approximate projection. In practice we only run $1$ epoch of \EP{2.0} for large scale experiments for which $T_{\rm ep2}=O(p^2)$. Similarly, $M_{\rm ep2}=O(p)$ is the memory rquired for running \EP{2.0}. Cost of kernel evaluations and number of classes are assumed to be $O(1)$.
} 

\end{table*}

\Cref{alg:eigenpro30exact} does not scale well to large models and large datasets, since it requires $O(np\vee p^2)$ memory and $O(np\vee p^3)$ FLOPS. We now propose  stochastic approximations that drastically make it scalable to both large models as well as large datasets.

\section{Upscaling via stochastic approximations}
\label{sec:speedup}

\begin{algorithm}[t]

\caption{\EP{3.0}}\label{alg:eigenpro30}
\begin{algorithmic}[1]
\Require Data $(X,\y)$, centers $Z$, batch size $m$, Nystr\"om size $s,$, preconditioner level $q.$
\State Fetch subsample $X_s\subseteq X$ of size $s$
\State $(\Lambda,\E,\lambda_{q+1})\gets$ top-$q$ eigensystem of $K(X_s,X_s)$
\State $\C\!=\! K(Z,X_s)\E(\Lambda^{-\!1}\!-\!\lambda_{q\!+\!1}\Lambda^{-\!2})\E\tran\in\Real^{p\times s}$

\While{Stopping criterion is not reached}
\State Fetch minibatch $(X_m,\y_m)$
\State $\g_m \gets K(X_m, Z)\alphavec - \y_m$ 
\State $\h \gets K(Z,X_m)\g_m-\C K(X_s,X_m)\g_m$
\State $\thetavec\leftarrow$ {\sf EigenPro\,2.0}$(Z,\h)$
\State $\alphavec\gets\alphavec- \frac{n}m\eta\,\thetavec$ \EndWhile
\end{algorithmic}
{\small 
Note: \EP{2.0}$(Z,\h)$ solves  $K(Z,Z)\thetavec=\h$ approximately\\
See Table \ref{tab:computational complexity} for a comparison of costs.
}
\end{algorithm}

\Cref{alg:eigenpro30exact} suffers 
from 3 main issues. It requires ---  (i)  access to entire dataset of size $O(n)$ at each iteration, (ii) $O(n^2)$ memory to calculate the preconditioner $\Q$, and (iii) $O(p^3)$ for the matrix inversion corresponding to an exact projection. This prevents scalability to large $n$ and $p$.

In this section we present 3 stochastic approximation schemes --- stochastic gradients, Nystr\"om approximated preconditioning, and inexact projection --- that drastically reduce the computational cost and memory requirements. These approximations together give us \Cref{alg:eigenpro30}.

\Cref{alg:eigenpro30exact} emulates \cref{eq:PGM}, whereas \Cref{alg:eigenpro30} is designed to emulate its approximation,
\begin{align}
    f^{t+1} = f^t-\tfrac{n}{m}\eta\cdot\wt{\text{proj}}_{\mathbb{Z}}\round{\mc P_s\curly{\wt\nabla_fL(f^t)}},
\end{align}
where $\wt\nabla_fL(f^t)$ is a stochastic gradient obtained from a sub-sample of size $m,$ $\mc P_s$ is a preconditioner obtained via a Nystr\"om extension based preconditioner from a subset of the data of size $s$, and $\wt{\text{proj}}_{\mathbb{Z}}$ is an inexact projection performed using \EP{2.0} to solve the projection equation $K(Z,Z)\thetavec=\h$.

\textbf{Stochastic gradients: }
We can replace the gradient with stochastic gradients, whereby
$\wt\nabla_fL(f^t)$ only depends on a batch $(X_m,\y_m)$ of size $m,$ denoted $X_m=\{\x_{i_j}\}_{j=1}^m$ and $\y_m=(y_{i_j})\in\Real^m$, 
\begin{align}\label{eq:stochastic_gradient}
    \wt\nabla_fL(f^t)
    &=K(\cdot, X_m) (K(X_m,Z)\alphavec-\y_m)\in\Xspan.
\end{align}

\begin{remark}
Here we need to scale the learning rate by $\frac{n}{m},$ to get unbiased estimates of $\nabla_fL(f^t)$.
\end{remark}

\textbf{Nystr\"om preconditioning: }
Previously, we obtained the preconditioner $\mc P$ from \cref{eq:preconditioner}, which requires access to all samples. We now use the Nystr\"om extension to approximate this preconditioner, see \citep{williams2000using}.
Consider a subset of size $s$, $X_s=\curly{\x_{i_k}}_{k=1}^s\subset X$. We introduce the Nystr\"om preconditioner,
\begin{align}
    \mc P_s:=\mc I-\sum_{i=1}^s\round{1-\frac{\lambda^{s}_{q+1}}{\lambda^{s}_i}}\psi^{s}_i\otimes \psi_i^{s}.
\end{align}
where $\psi^{s}_i$ are eigenfunctions of $\mc K^s:=\sum_{k=1}^s K(\cdot,\x_{i_k})\otimes K(\cdot,\x_{i_k})$. Note that $\mc K^s\approx\frac{s}{n}\mc K$ since both approximate $\mc T_K$ as shown in \cref{eq:integral_operator}. This preconditioner was first proposed in \citep{ma2019kernel}.

Next, we must understand the action of $\mc P_s$ on elements of $\Xspan$.
Let $(\Lambda_{q},\E_{q},\lambda_{q+1})$ be the top-$q$ eigensystem of $K(X_s, X_s)$. Define the rank-$q$ matrix,
\begin{align}
    \Q_s:=\E_{q}(\I_s-\lambda_{q+1}\Lambda_{q}\inv)\Lambda_{q}\inv\E_{q}\tran\in\Real^{s\times s}.
\end{align}

\begin{lemma}[Nystr\"om preconditioning]\label{lemma:Nystrom preconditioning}
Let $\a\in\Real^m$, and $X_m$ chosen like in \cref{eq:stochastic_gradient}, then, 
\begin{align*}
    \mc P_s\!\curly{ K(\cdot,X_m\!) \a}\!=\! K(\cdot,X_m\!)\a
    \!-\!  K(\cdot,X_s\!)\Q_s K(X_s,\!X_m\!) \a.
\end{align*}
\end{lemma}
Consequently, using \cref{eq:stochastic_gradient}, we get,
\begin{align}
    &\mc P_s\curly{\wt\nabla_{\!f}L(f^t)}(Z)=\nonumber\\
    &\quad \Big(K(Z,X_m)-  K(Z,X_s)\Q_s K(X_s,X_m)\Big)\times\nonumber\\
    &\qquad\big(K(X_m,Z)\alphavec^t-\y_m\big)\in\Real^p.
\end{align}

\textbf{Inexact projection: } 
The projection step in \Cref{alg:eigenpro30exact} requires the inverse of $K(Z,Z)$ which is computationally expensive. However this step is solving the $p\times p$ linear system
\begin{align}\label{eq:approx_projection}
    &K(Z,Z)\thetavec=\h\\  
    &\h:=\Big(K(Z,X_m)-  K(Z,X_s)\Q_s K(X_s,X_m)\Big)\g \nonumber\\
    &\g:=\big(K(X_m,Z)\alphavec^t-\y_m\big).\nonumber
\end{align}

\begin{figure*}[th]
        \centering 
        \begin{subfigure}{1\textwidth}
        \includegraphics[width=1.0\linewidth]{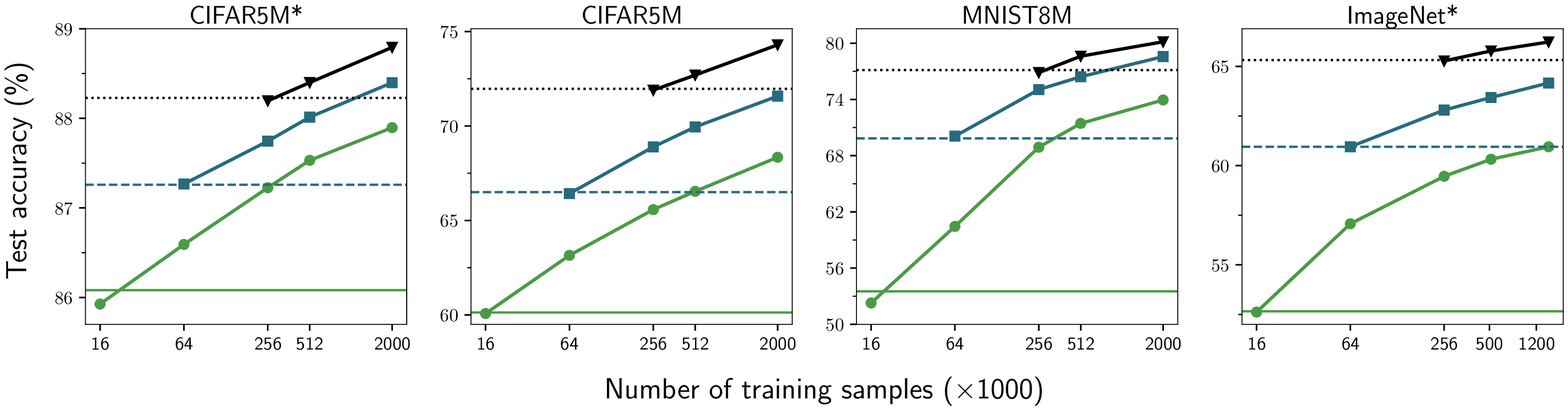}%
        \end{subfigure}
        
        \begin{subfigure}{1\textwidth}

        \hspace{0.25cm}
        \includegraphics[width=1.0\linewidth]{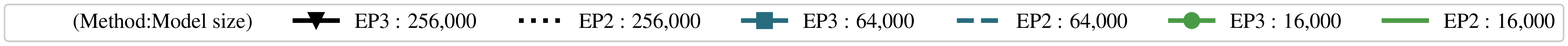}%
        \end{subfigure}
        
        \vspace{-0.3cm}
        \caption{\textbf{(Scaling number of training samples)} 
        Model centers are selected by random sub-sampling from the training data set.  The baselines (lines without markers) are obtained from a standard kernel machine solved by \EP{2.0} over the centers and their corresponding labels. 
        Lines with markers indicate the performance of kernel models trained with our algorithm (\EP{3.0})  after 50 epochs. 
        }
    \label{fig:accvsn:main}
\end{figure*}

\begin{figure*}[th]
        \centering      
        \begin{subfigure}{1\textwidth}
        \includegraphics[width=1.0\linewidth]{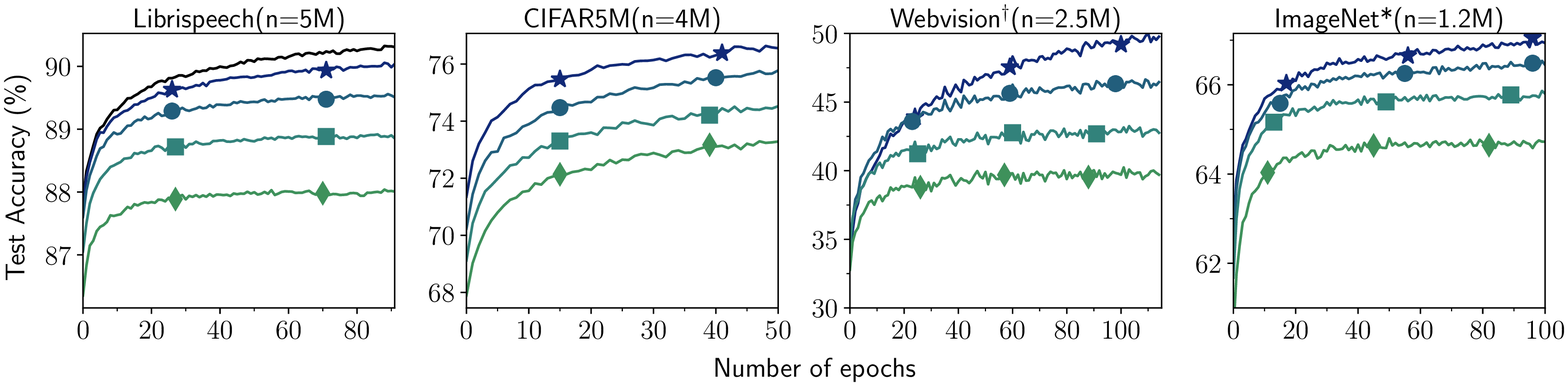}%
        \end{subfigure}

        \begin{subfigure}{1\textwidth}
        \hspace{1.1cm}
        \includegraphics[width=0.9\linewidth]{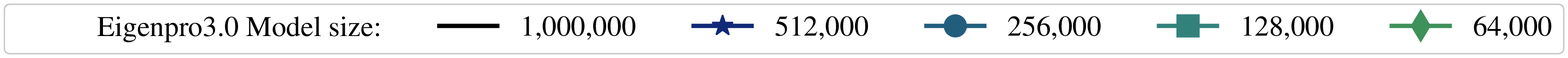}%
        \end{subfigure}
        \vspace{-20pt}
        \caption{(\textbf{Scaling model size}) Performance of \EP{3.0} for different number of model centers,  fixed number of data $n$.
        }

    \label{fig:bigmodels}
\end{figure*}

Notice that this is the kernel interpolation problem \EP{2.0} can solve.
This leads to the update,
\begin{align*}
    &\alphavec^{t+1}=\alphavec^t-\frac{n}{m}\eta\,\wh\thetavec^T\tag{\EP{3.0} update}
\end{align*}
where $\wh\thetavec^T$ is the solution to \cref{eq:approx_projection} after $T$ steps of \EP{2.0} given in \Cref{alg:eigenpro20} in the Appendix.
\Cref{alg:eigenpro30}, \EP{3.0}, implements the update above.

\begin{remark}[Details on inexact-projection using \EP{2.0}]
We apply $T$ steps of \EP{2.0} for the approximate projection. This algorithm itself applies a fast preconditioned SGD to solve the problem. The algorithm needs no hyperparameters adjustment. However, you need to choose $s$ and $q$. More details on this in the Appendix \ref{choiceofhypers}.
\end{remark}

\begin{remark}[Decoupling]
    There are two preconditioners involved in \EP{3.0}, a data preconditioner (for the stochastic gradient) which depends only on $X$, and a model preconditioner (for the inexact projection) which depends only on $Z$. This maintains the models decoupling from the training data.
\end{remark}

        

        




\textbf{Complexity analysis:}
We compare the complexity of the run-time and memory requirement of \Cref{alg:eigenpro30} and \Cref{alg:eigenpro30exact} with \FALKON~solver in \Cref{tab:computational complexity}.

\section{Real data experiments}\label{sec:expts}

\normalsize In this section, we demonstrate that our method can effectively scale both the size of the model and the number of training samples. We show that both of these factors are crucial for a better performance. We perform experiments on these datasets: (1) CIFAR10, CIFAR10\footnote[1]{\label{mobilenet2}feature extraction using MobileNetV2} \citep{krizhevsky2009learning}, (2) CIFAR5M, CIFAR5M$^*$ \citep{nakkiran2020deep}, (3) ImageNet$^*$, \citep{deng2009imagenet}, (4) MNIST, \citep{lecun1998mnist},  (5) {MNIST8M}, \citep{loosli2007training}, (6) Fashion-MNIST, \citep{xiao2017fashion}, (7) Webvision\footnote[2]{feature extraction using ResNet-18}, \citep{li2017webvision}, and (8) Librispeech, \citep{panayotov2015librispeech}. Details about datasets can be found in Appendix \ref{sec:expt_details}. Our method can be implemented with any kernel function, but for the purpose of this demonstration, we chose to use the Laplace kernel due to
its simplicity and empirical effectiveness. We treat multi-class classification problems as multiple independent binary regression problems, with targets from $\curly{0,1}$. The final prediction is determined by selecting the class with the highest predicted value among the $K$ classes.

\begin{figure*}[th]
  \centering
    \includegraphics[width=0.8\linewidth]{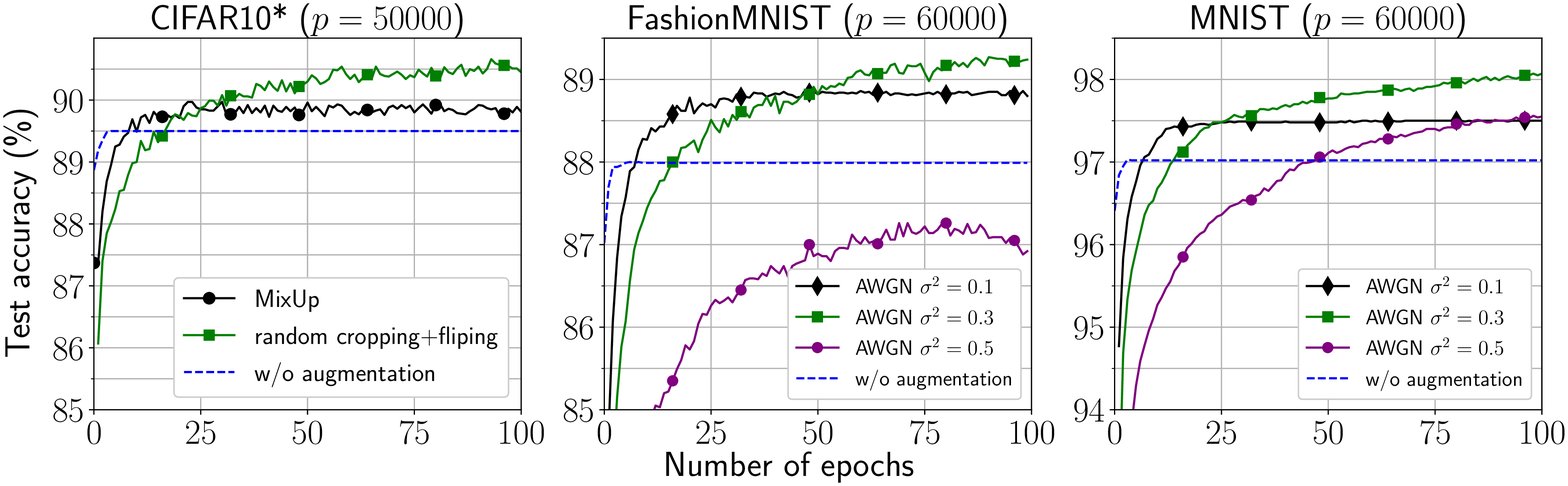}
  \label{fig:CIFAR10_AUGMENT}
  \vspace{-2mm}
    \caption{\label{fig:augmentation}(\textbf{Data augmentation for kernel models}) Entire original dataset was used as as the model centers $Z$. The model was trained using the augmented set $X$ (without the original data). MixUp and Crop+Flip augmentation was used for   CIFAR10. Additive White Gaussian Noise (AWGN) augmentation was used for MNIST and FashionMNIST.
    }
\end{figure*}

\begin{figure}[th]
    \vspace{-20pt}
        \centering        \includegraphics[width=0.9\linewidth]{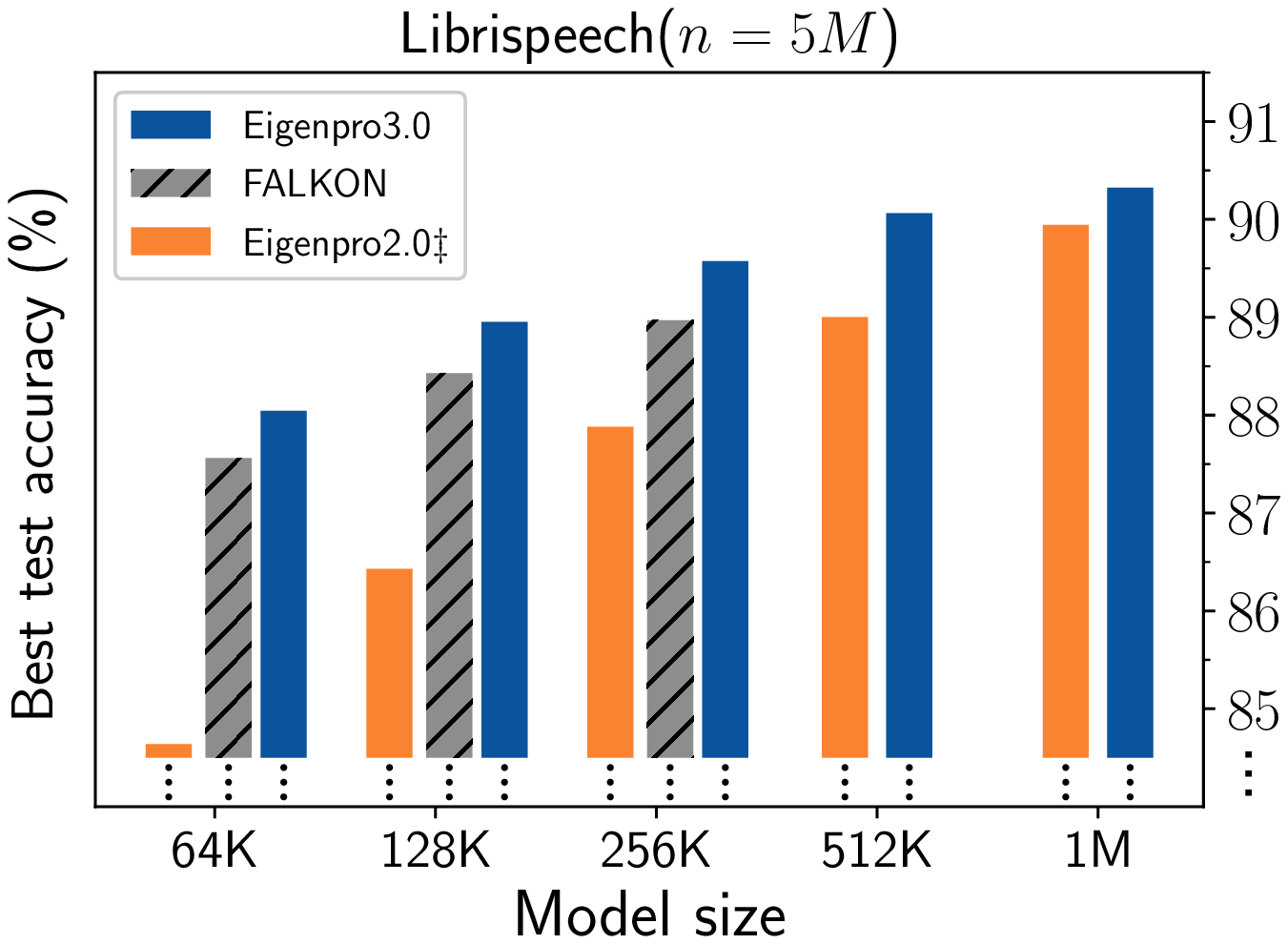}%
     \vspace{-10pt}   
        \caption{ Comparison of \EP{3.0} (EP3) performance with  \FALKON~ and \EP{2.0} (EP2).\\
        $^\ddagger$ For \EP{2.0} models size and number of training samples are the same. $n=5$M was used for all other methods. FALKON could not be run for model size larger than $256$K centers due to memory limitations. For each model size, \EP{2.0} and \EP{3.0} use the same centers. 
        }
    \vspace{-10pt}
    \label{fig:falkonvsep2vsep3}
\end{figure}

\textbf{Scaling the number of training samples: }
Figure \ref{fig:accvsn:main} illustrates that for a fixed model size, increasing the number of training samples leads to a significant improvement in performance, as indicated by the the lines with marker. As a point of comparison, we also include results from a standard kernel machine using the same centers, represented by the horizontal lines without markers. In this experiment, the centers were randomly selected from the data and they are a subset of training data.

\textbf{Scaling the model size: } Figure \ref{fig:bigmodels} shows that, when the number of training samples is fixed, larger models perform significantly better. 

To the best of our knowledge, \FALKON ~ is the only method capable of training general kernel models with model size larger than $100,\!000$ centers. Our memory resources of $340$GB RAM allowed us to handle up to $256,\!000$ centers using \FALKON. Figure \ref{fig:falkonvsep2vsep3} illustrates that our method outperforms \FALKON ~ by utilizing larger models. Moreover, for a fixed set of centers, training on greater number of data will boost the performance. Therefore, \EP{3.0} outperforms \EP{2.0} by training on more data.

\textbf{Data augmentation for kernel models: }
Data augmentation is a crucial technique for improving the performance of deep networks. We demonstrate its benefits for general kernel models by selecting the model centers $Z$ to be the original training data and the train set $X$ to be the virtual examples generated through augmentation. To the best of our knowledge, this is the first implementation of data augmentation for kernel models at this scale. Figure \ref{fig:augmentation} shows we have significant improvements in accuracy. 


\textbf{Data processing details: } We performed our experiment on CIFAR10 with feature extraction, raw images of MNIST and FashionMNIST. For CIFAR10 augmentation, we apply random cropping and flipping before feature extraction. We also apply mix-up augmentation method from \citep{zhang2017mixup} after feature extraction. For MNIST and FashionMNIST augmentation we added Gaussian noise with different variances. To the best of our knowledge, this is the first implementation of data augmentation for kernel models at this scale. Figure \ref{fig:augmentation} shows we have significant improvements in accuracy.

\section{Conclusions and Outlook}
\label{sec:conclusions}
\begin{table*}[h]
    \centering

    \begin{tabular}{cccccc}\toprule
    Dataset
    & Model
    & $p=100$ 
    & $p=1000$ 
    & $p=10000$ 
    \\\midrule
    CIFAR10
    & $k-$means 
    & \textbf{\small 36.24}
    & \textbf{\small 45.12}
    & \textbf{\small 52.72}
    \\
    ($n=50000$)
    & random 
    &{\small$33.37\pm0.50$}
    &{\small$44.19\pm0.09$}
    &{\small$49.92\pm0.08$}

    \\\midrule
    CIFAR10*
    & $k-$means 
    &{\small\textbf{82.69}}
    &{\small\textbf{86.58}}
    &{\small\textbf{89.11}}
    \\
    ($n=50000$)
    & random 
    &{\small$74.29\pm0.44$}
    &{\small$84.38\pm0.15$}
    &{\small$86.58\pm0.06$}
    \\\midrule 
    MNIST
    & $k-$means 
    &{\small$\boldsymbol{91.89}$}
    &{\small$\boldsymbol{95.96}$}
    &{\small$\boldsymbol{97.69}$}
    \\
    ($n=60000$)
    & random 
    &{\small$87.24\pm0.015$}
    &{\small$94.96\pm0.102$}
    &{\small$97.31\pm0.004$}
    \\\midrule 
    FashionMNIST
    & $k-$means 
    &{\small$\boldsymbol{78.66}$}
    &{\small$\boldsymbol{85.55}$}
    &{\small$\boldsymbol{88.13}$}
    \\
    ($n=60000$)
    & random 
    &{\small$76.24\pm0.003$}
    &{\small$84.59\pm0.069$}
    &{\small$87.84\pm0.036$}

    \\\bottomrule
    \end{tabular}
    \caption{ \label{tab:flexible}(\textbf{Benefits of model flexibility}) 
    Comparison between random centers selection and k-means clustering using \EP{3.0}. Here $p$ denotes the number of centers.\\
    ($^*$ indicates a preprocessing step.)}
 
\end{table*}

The remarkable success of Deep Learning has been in large part due to very large models trained on massive datasets. Any credible alternative requires a  path to scaling  model sizes as well as the size of the training set.
Traditional kernel methods suffer from a severe scaling limitation as their model sizes are coupled  to size of the training set. Yet, as we have seen in numerous experiments, performance improves with the amount of data far beyond the point where the amount of data exceeds the number of model centers. Other solvers, such as \FALKON~\citep{rudi2017falkon}, \GPytroch~ \citep{gardner2018product} \textsf{\small GPFlow} \cite{GPflow2017}, allow for unlimited data but limit the model size.  

In this work we provide a proof of concept showing that for kernel methods the barrier of limited model size can be overcome. 
Indeed, with a fixed model size {our proposed algorithm,} \EP{3.0} has no specific limitations on the number of samples it can use in training. As a simple illustration  Fig.~\ref{fig:100Mdata} demonstrates the results of a kernel machine trained on an augmented FashionMNIST dataset with $1.2\!\times\!10^8$ data samples.
While increasing the model size is more challenging, we have achieved 1 million centers and see no fundamental mathematical barrier  to increasing 
the number of centers to 10 million and beyond. Furthermore, as \EP{3.0} is highly parallelizable, we anticipate future scaling to tens of millions of centers trained on billions of data points, approaching the scale of modern neural networks.

\begin{figure}[th]
    \vspace{-20pt} 
    \centering      
    \includegraphics[width=0.95\linewidth]{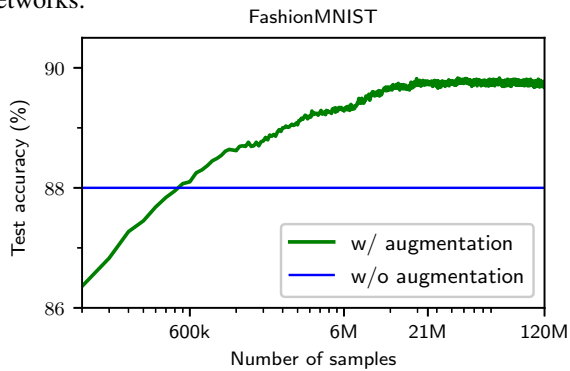}%
     \vspace{-5pt}   
        \caption{(\textbf{Training on 120 million samples)} The plot shows the results of training on 120 million data samples generated by adding Gaussian noise to the original 60,000 images. The model centers are the images.
        }
    \vspace{-10pt}
    \label{fig:100Mdata}
\end{figure}

This line of research opens a pathway for a principled alternative to deep neural networks. While in our experiments we focused primarily on Laplace kernels for their simplicity and effectiveness, recently developed kernels such as NTK, CNTK, and other neural kernels from \citep{shankar2020neural}, can be used to achieve state-of-the-art performance on various datasets. Our approach is compatible with any choice of kernel, and furthermore, can be adapted to kernels that learn features, such as Recursive Feature Machines~\citep{radhakrishnan2022feature}.

\vspace{-0.2cm}

Finally, we note that while the set of centers $Z$ can be arbitrary, the choice of centers can affect model performance. Table~\ref{tab:flexible} (see also~\citep{que2016back}) demonstrates that using centers selected via the $k$-means algorithm often results in notable performance improvements when the number of centers is smaller than the number of samples. Further research should explore criteria for optimal center selection, incorporating recent advances, such as Rocket-propelled Cholesky~\citep{chen2022randomly},  into \EP{3.0} to improve both model selection and construction of preconditioners.

\subsubsection*{Acknowledgments}
We thank Giacomo Meanti for help with using \FALKON, and Like Hui for  providing us with the extracted features of the Librispeech dataset. \\
We are grateful for the support from the National Science Foundation (NSF) and the Simons Foundation for the Collaboration on the Theoretical Foundations of Deep Learning (\url{https://deepfoundations.ai/}) through awards DMS-2031883 and \#814639  and the TILOS institute (NSF CCF-2112665). This work was done in part while the authors were visiting the Simons Institute for the Theory of Computing. This work used NVIDIA V100 GPUs NVLINK and HDR IB (Expanse GPU) at SDSC Dell Cluster through allocation TG-CIS220009 and also, Delta system at the National Center for Supercomputing Applications through allocation bbjr-delta-gpu from the Advanced Cyberinfrastructure Coordination Ecosystem: Services \& Support (ACCESS) program, which is supported by National Science Foundation grants \#2138259, \#2138286, \#2138307, \#2137603, and \#2138296.
Prior to 09/01/2022, we used the Extreme Science and Engineering Discovery Environment (XSEDE) \citep{XSEDE}, which is supported by NSF grant number ACI-1548562, Expanse CPU/GPU compute nodes, and allocations TG-CIS210104 and TG-CIS220009.

\clearpage

\bibliographystyle{icml2023.bst}
\bibliography{aux/ref}

\onecolumn
{\Huge\sc Appendices}
\appendix
\begin{table}[h]
    \caption{Symbolic notation for \EP{3.0} in \Cref{alg:eigenpro30}. They satisfy $m<n$, and $q<s<n.$}
    \label{tab:symbolic}
\centering
    \begin{tabular}{|c|c|}
    \hline
    Symbol & Purpose \\\hline
     $n$ & Number of samples  \\
     $m$ & Batch-size\\
     $p$ & Model size\\
     $s$ & Nystr\"om approximation subsample size\\
     $q$ & Preconditioner level\\
     \hline
\end{tabular}
\end{table}

\section{Fixed point analysis}

Here we provide a characterization of the fixed point of   \Cref{alg:eigenpro30exact}.
\begin{lemma}
    For any dataset $X,\y$ and any choice of model centers $Z$, if the learning rate satisfies
    $$\eta < \frac{2}{\lambda_{\rm max}\round{K(Z,X)(I_n-\Q)K(X,Z)}}$$
    we have that
    \begin{align*}
        \lim_{t\rightarrow\infty}\alphavec_t = \round{K(Z,X)(I_n-\Q)K(X,Z)}\inv K(Z,X)(I_n-\Q)\y.
    \end{align*}
    Furthermore, if $\y=K(X,Z)\alphavec^*+\xivec$, where $\xi_i$ are independent centered with $\Exp|\xi_i|^2=\sigma^2$, then \begin{align*}
        \lim_{t\rightarrow\infty}\Exp\alpha_t&=\alphavec^*\\
        \lim_{t\rightarrow\infty}\frac{\Exp\norm{\alphavec_t-\alphavec^*}^2}{\sigma^2} 
    &= {\rm tr}\round{\round{K(Z,X)(I_n-\Q)K(X,Z)}^{-2}K(Z,X)(I_n-\Q)^2K(X,Z)}\\
    &= n - {\rm trace}\round{\round{K(Z,X)(I_n-\Q)K(X,Z)}^{-2}K(Z,X)\Q(I_n-\Q)K(X,Z)}
    \end{align*}
\end{lemma}
\section{Proofs of intermediate results}
\subsection{Proof of proposition \ref{prop:nystrom_extension}}

\newtheorem*{proposition*}{Proposition}
\begin{proposition*}
[Nystr\"om extension]
For $1\leq i\leq n$, let $\lambda_i$ be an eigenvalue of $\mc K$, and $\psi_i$ its unit $\Hilbert$-norm eigenfunction, \ie, $\mc K\curly{\psi_i}=\lambda_i\psi_i$. Then $\lambda_i$ is also an eigenvalue of $K(X,X)$. Moreover if $\e_i,$ is its unit-norm eigenvector, \ie, $K(X,X)\e_i=\lambda_i\e_i$, we have,
\begin{align}
 \psi_i  = K(\cdot,X)\frac{\e_i}{\sqrt{\lambda_i}}.
\end{align}
\end{proposition*}

\begin{proof}

Let $\psi \in \mc H$ be an eigenfunction of $\Kcov$. Then by definition of $\Kcov$ we have,
\begin{equation}\label{appx:eigenvec1}
    \lambda\psi=\Kcov\curly{\psi}  =  \sum_{i=1}^{n}  K(\cdot,\x_i)\psi(\x_i).
\end{equation}
As the result we can write $\psi$ as below,
\begin{equation}\label{appx:eigenvec2}
    \psi= \sum_{i=1}^{n} \frac{\psi(\x_i)}{\lambda}  K(\cdot,\x_i).
\end{equation}
If we apply covariance operator to the both side of \ref{appx:eigenvec2} we have,
\begin{align}\label{appx:eigenvec3}
        \Kcov\curly{\psi}&= \Kcov\curly{\sum_{i=1}^{n} \frac{\psi(\x_i)}{\lambda}K(\cdot,\x_i) }=\sum_{i,j=1}^n  \frac{\psi(\x_i)}{\lambda} K(\x_i,\x_j)K(\cdot,\x_j)=\sum_{j=1}^n \psi(\x_j)K(\cdot,\x_j).
\end{align}
The last equation hold because of equation \eqref{appx:eigenvec1}. If we define vector $\betavec$ such that $\betavec_i=\frac{\psi(x_i)}{\lambda}$, then \ref{appx:eigenvec3} can be rewritten as,
\begin{align}\label{appx:eigenvec4}
        \sum_{i=1}^n \sum_{j=1}^n \betavec_i K(\x_i,\x_j)K(\cdot,\x_i)=\lambda \sum_{i=1}^n \betavec_i K(\cdot,\x_i).
\end{align}
Compactly we can write \ref{appx:eigenvec4} as below,
\begin{equation*}
K(X,X)^2 \betavec = \lambda K(X,X)\betavec \implies K(X,X) \betavec =\lambda\betavec.
\end{equation*} 
The last implication holds because $K(X,X)$ is invertable. Thus $\betavec$ is an eigenvector of $K(X,X)$. It remains to determine the scale of $\betavec$.

Now, norm of $\psi$ can be simplified as
\begin{align}
 \norm{\psi}_\Hilbert^2 
 &= \inner{\sum_{i=1}^n \beta_i K(\cdot,\x_i),\sum_{j=1}^n \beta_j K(\cdot,\x_j)}_\Hilbert \\
 &= \sum_{i,j=1}^n \beta_i\beta_j\inner{K(\cdot,\x_i),K(\cdot,\x_j)}_\Hilbert =
 \betavec\T K(X,X)\betavec=\lambda\norm{\betavec}^2.
\end{align}
Since $\psi$ is unit norm, we have $\norm{\betavec}=\frac{1}{\sqrt{\lambda}}$. This concludes the proof.
\end{proof}

\subsection{Proof of lemma \ref{lemma:Nystrom preconditioning}}

\newtheorem*{lemma*}{Lemma}

\begin{lemma*}[Nystr\"om preconditioning]\label{appendix:lemma:Nystrom preconditioning}
Let $\a\in\Real^m$, then we have that, 
\begin{align}
    \mc P_s\curly{ K(\cdot,X_m) \a}= K(\cdot,X_m)\a-  K(\cdot,X_s)\Q_s K(X_s,X_m) \a.
\end{align}

Where $Q_s =E_{s,q}(\I_n-\lambda_{s,q+1}\Lambda_{s,q}^{-1})\Lambda_{s,q}^{-1}E_{s,q}\T $.
\end{lemma*}

\begin{proof}
    Recall that $\mc P_s:=\mc{I} - \sum_{i=1}^q \round{1-\frac{\lambda_{q+1}}{\lambda_{q}}} \psi_i\otimes \psi_i$. By this definition we can write,

    \begin{align*}
    \Precond_{s}\rbrac{K(\cdot,X_M)\alphavec} 
    &= K(\cdot,X_M)\alphavec - \sum_{i=1}^{s}(1-\frac{\lambda_{q+1}^s}{\lambda_i^s})\inner{ \psi_i^s,K(\cdot,X_M)\alphavec }_{\Hilbert} \psi_i^{s}\\
    &=K(\cdot,X_M)\alphavec - \sum_{i=1}^{q}\frac{1}{\lambda_i^s}(1-\frac{\lambda_{q+1}^s}{\lambda_i^s}) \inner{ 
    {K(\cdot,X_s)}\e_i, K(\cdot,X_M)\alphavec
    }_{\Hilbert} K(\cdot,X_s)\e_i\\
    &=K(\cdot,X_M)\alphavec - \sum_{i=1}^{q}\frac{1}{\lambda_i^s}(1- \frac{\lambda_{q+1}^s}{\lambda_i}) \inner{ 
    { K(\cdot,X_s)\e_i}, K(\cdot,X_M)\alphavec
    }_{\Hilbert}\!  K(\cdot,X_s)\e_i
    \\
    &=K(\cdot,X_M)\alphavec - \sum_{i=1}^q(1- \frac{\lambda_{q+1}^s}{\lambda_i^s})K(\cdot,X_s)\e_i\e_i\T K(X_s,X_M) \alphavec .
    \end{align*}
Note that we used proposition \ref{prop:nystrom_extension} for $\psi$. Now we can compactly write the last expression as below,

\begin{align*}
    \Precond_{s}\rbrac{K(\cdot,X_M)\alphavec} &=
    K(\cdot,X_M)\alphavec-K(\cdot,X_s)E_{s,q}(\I_n-\lambda_{s,q+1}\Lambda_{s,q}^{-1})\Lambda_{s,q}^{-1}E_{s,q}\T K(X_s,X_M)\alphavec\\
    &=
    K(\cdot,X_M)\alphavec-K(\cdot,X_s)Q_s K(X_s,X_M)\alphavec.
\end{align*}

This concludes the proof.

\end{proof}

\section{Details on \texorpdfstring{\sf EigenPro\,2.0}{EigenPro2.0}}

\begin{algorithm}[t]
\caption{\textsf{EigenPro\,2.0}$(X,\y)$. Solves the linear system $K(X,X)\thetavec=\y$}\label{alg:eigenpro20}
\begin{algorithmic}
\Require Data $(X,\y)$, Nystr\"om size $s$, preconditioner level $q$
\State $\alphavec\gets\textbf{0}\in\Real^n$\Comment{initialization}
\State $X_s,(\E,\D), \eta, m \gets $ \textsf{EigenPro\,2.0\_setup}($X,s,q$)
\While{Stopping criterion not reached}
\State $\alphavec \gets $ \textsf{EigenPro\,2.0\_iteration}$( X,\y, X_s,\E,\D,\alphavec,m, \eta)$
\EndWhile
\State \Return $\alphavec$
\end{algorithmic}
\vspace{3mm}
\textsf{EigenPro2\_setup}$(X,s,q)$
\begin{algorithmic}
\Require Data $X$, Nystr\"om size $s$, preconditioner size $q$
\State Fetch a subsample $X_s\subseteq X$ of size $s$
\State $(\E,\Lambda_q,\lambda_{q+1}) \gets $ top-$q$ eigensystem of $K(X_s, X_s)$\Comment{$\E\in\Real^{q\times s},\Lambda=\textrm{diag}(\lambda_i)\in\Real^{q\times q}$}
\State $\D_{ii}=\frac{1}{s\lambda_i}\round{1-\frac{\lambda_{q+1}}{\lambda_i}}$
\State $\beta\gets \max{i}K(\x_i,\x_i) \in S$
\State $m\gets {\rm min}\round{\frac{\beta}{\lambda_{q+1}},{\sf bs}_{\rm gpu}}$\Comment{batch size\footnotemark}
\State $\eta\gets\begin{cases}\frac{\beta}{2m} & m<\frac{\beta}{\lambda_{q+1}}\\
\frac{0.99m}{\beta+(m-1)\lambda_{q+1}} & {\rm otherwise}
\end{cases}$ 
\Comment{learning rate}
\State \Return
$X_s,(\E,\D),\eta,m$
\end{algorithmic}
\vspace{3mm}
\textsf{EigenPro2\_iteration}$(X,\y,X_s,\E,\D,\alphavec,m,\eta)$
\begin{algorithmic}
\Require Data $(X,\y)$, Nystr\"om subset $X_s$,  preconditioner $(\E,\D)$, current estimate $\alphavec$, batchsize $m$
\State Fetch minibatch $(X_m,\y_m)$ of size $m$ 
\State $\g_m \leftarrow K(X_m, X)\alphavec-\y_m$ \Comment{stochastic gradient}
\State $\alphavec_m \gets \alphavec_m-\frac{\eta}m \g_m$ \Comment{gradient step}
\State $\alphavec_s\gets\alphavec_s+\E\D\E\tran K(X_s,X_m)\g_m$ \Comment{gradient correction}
\State \Return Updated estimte $\alphavec$
\end{algorithmic}
\end{algorithm}

\begin{lemma}\label{lem:EigenPro_Rn_equivalence}
The iteration in $\Real^n$
\begin{align}
    \alphavec^{t+1}=\alphavec^{t+1}-\eta(\I_n-\Q)(K(X,X)\alphavec^t-\y),
\end{align}
where $Q=\E(\I_n-\lambda_{q+1}\Lambda_q^{-1})\E\T$, emulates the following iteration in $\Hilbert.$ 
\begin{align}\label{appendix:gdinhilbertspace}
f^{t+1}=f^{t} - \eta\mc P\curly{\nabla_f L(f^t)}.
\end{align}
\end{lemma}
\begin{proof}
Recall that $\nabla_f L(f^t)=K(\cdot,X)(f^t(X)-\y)$ from
\cref{eq:frechet_square_loss}, and $f^t(X)=K(X,X)\alphavec^t.$ from \cref{eq:evaluation}. We define $\g^t:=f^t(X)-\y=K(X,X)\alphavec^t-\y$. Following steps of the proof in \Cref{appendix:lemma:Nystrom preconditioning} we have
\begin{align*}
    \Precond\{\nabla_f L(f^t)\} &= K(\cdot,X)\g^t- \sum_{i=1}^q(1- \frac{\lambda_{q+1}}{\lambda_i})K(\cdot,X)\e_i\T\e_i K(X,X) \g^t\\
    &=K(\cdot,X)\g^t - K(\cdot,X)\E(\I_n-\lambda_{q+1}\Lambda_q^{-1})\Lambda^{-1}\E\T K(X,X)\g^t\\
    &\overset{\rm (a)}=K(\cdot,X)\g^t - K(\cdot,X)\E(\I_n-\lambda_{q+1}\Lambda_q^{-1})\Lambda^{-1}\E\T \E\Lambda\E\T\g^t\\
    &=K(\cdot,X)\g^t - K(\cdot,X)\E(\I_n-\lambda_{q+1}\Lambda_q^{-1})\E\T\g^t\\
    &= K(\cdot,X)\g^t -K(\cdot,X)\Q\g^t\\
    &=K(\cdot,X)(\I_n-\Q)\g^t.
\end{align*}
Where $(a)$ follows from $ K(X,X)=\E\Lambda\E\tran$.
Now since $f^t = K(\cdot,X)\alphavec^t$, equation \eqref{appendix:gdinhilbertspace} can be rewritten,
\begin{align*}
    f^{t+1}&=K(\cdot,X)\alphavec^{t+1} - \eta K(\cdot,X)(\I_n-Q)\g^t\\
    &=K(\cdot,X)(\alphavec^{t+1}-\eta(\I_n-Q)\g^t).
\end{align*}
Replacing $g^t = K(X,X)\alphavec^t-y$ leads to final update rule below,
\begin{align*}
    f^{t+1}=K(\cdot,X)(\alphavec^{t+1}-\eta(\I_n-Q)(K(X,X)\alphavec^t-y)).
\end{align*}

This concludes the proof.
\end{proof}
\footnotetext{${\sf bs}_{\rm gpu}$ is the maximum batch-size that the GPU allows.}

Thus each update constitutes a \textit{stochastic gradient step} which consists updating $m$ weights corresponding to a minibatch size $m$, followed by a \textit{gradient correction} which consists of updating all $n$ weights.

A higher preconditioner level $q$ also allows for a higher optimal batch size $m$ and hence better GPU utilization, see \citet{ma2018power} for details.

With this approximation, the gradient correction simplifies drastically, and only $s$ weights need to be updated.

\section{Details on experiments and implementation of \texorpdfstring{\Cref{alg:eigenpro30}}{}}
\label{sec:expt_details}

\subsection{Computational resources used}
This work used the Extreme Science and Engineering Discovery Environment (XSEDE) \citep{XSEDE}. We used machines with 2x NVIDIA-V100 and 8x NVIDIA-A100 GPUs, with a V-RAM of 32GB and 40GB respectively, and 8x cores of Intel(R) Xeon(R) Gold 6248 CPU @ 2.50GHz with a RAM of 100 GB.

\subsection{Figure \ref{small_model} experiment}\label{sec:fig1_app} We used Laplacian Kernel and  (\texttt{sklearn.linear\_model.Ridge}) solver from the Scikit-learn library \cite{scikit-learn} to solve the optimization problem $\norm{K(X,Z)\alpha-y}^2 + \lambda\cdot\norm{\alpha}^2$ for extracted features of ImageNet, using a pre-trained MobileNetv2 model obtained from the \textit{timm} library \citet{rw2019timm}.

\subsection{Datasets}

We perform experiments on these datasets: (1) CIFAR10, 
\citet{krizhevsky2009learning}, (2) CIFAR5M,  
\citet{nakkiran2020deep}, (3) ImageNet, 
\citet{deng2009imagenet}, (4) MNIST, 
\citet{lecun1998mnist},  (5) {MNIST8M}, 
\citet{loosli2007training}, (6) Fashion-MNIST, 
\citet{HIGGS}, (7) Webvision.\citet{li2017webvision}, and (8) librispeech.

\paragraph{CIFAR5M.} In our experiments, we utilized both raw and embedded features from the CIFAR5M data-set. The embedded features were extracted using a MobileNetv2 model pre-trained on the ImageNet data-set, obtained from \textit{timm} library \citet{rw2019timm}. We indicate in our results when pre-trained features were used by adding an asterisk (*) to the corresponding entries.

\paragraph{ImageNet.} In our experiments, we utilized embedded features from the ImageNet data-set. The embedded features were extracted using a MobileNetv2 model pre-trained on the ImageNet dataset, obtained from \textit{timm} library \citet{rw2019timm}. We indicate in our results when pre-trained features were used by adding an asterisk (*) to the corresponding entries.

\paragraph{Webvision.} In our experiments, we utilized embedded features from the Webvision data-set. The embedded features were extracted using a ResNet-18 model pre-trained on the ImageNet dataset, obtained from \textit{timm} library \citet{rw2019timm}. Webvision data set contains 16M images in 5k classes. However, we only considered the first 2k classes.

\paragraph{Librispeech.}Librispeech \cite{panayotov2015librispeech} is a large-scale (1000 hours in total) corpus of 16 kHz English speech derived from audio books. We choose the subset train-clean-100 and train-clean-300 (5M samples) as our training data, test-clean as our test set. The features are got by passing through a well-trained acoustic model (a VGG+BLSTM architecture in \cite{hui2020evaluation} ) to align the length of audio and text. It is doing a 301-wise classification task where different class represents different uni-gram \cite{jurafsky2000speech}. The implementation of extracting features is based on the ESPnet toolkit \cite{watanabe2018espnet}.
\subsection{Choice of hyperparameters}\label{choiceofhypers}

We choose hyperparameters to minimize computation and maximize GPU utilization. The only hyperparameters  that we need to set are $s,q$ for outer gradient step, and $\sigma,\xi$ for projection sub-problem. For $\sigma,\xi$, we used the same criteria as \citet{ma2019kernel} to optimally use GPU utilization. For $s,q$, we prefer larger $q$ because as it is explained in \citet{ma2018power}, larger $q$ allows for larger learning rate and better condition number. However, in our algorithm we need to approximate the top $q$ eigensystem of Nystr\"om sub-samples matrix. We used Scipy \citet{2020SciPy-NMeth} library to approximate these eigensystem. The stability and precision of these approximations depends on how large is the ratio of $\frac{s}{q}$. Empirically we need this ratio to be larger than 10. On the other hand increasing $s$ will increase setup cost, computation cost and  memory cost. We take steps below to choose $q$ and $s$,
\begin{enumerate}
    \item We first choose $s$ as big as our GPU memory allow
    \item We choose $q \approx \frac{s}{10}$
    \item We set batch size and learning rate automatically using the \textit{new} top eigenvalue as it is explained in \citet{ma2019kernel} and \citet{ma2018power}.
\end{enumerate}


 
 

\section{Classical approach to learning kernel models with GD}\label{sec:classical}
\begin{figure*}[t]
        \centering 
        \begin{subfigure}{1\textwidth}
        \includegraphics[width=1.0\linewidth]{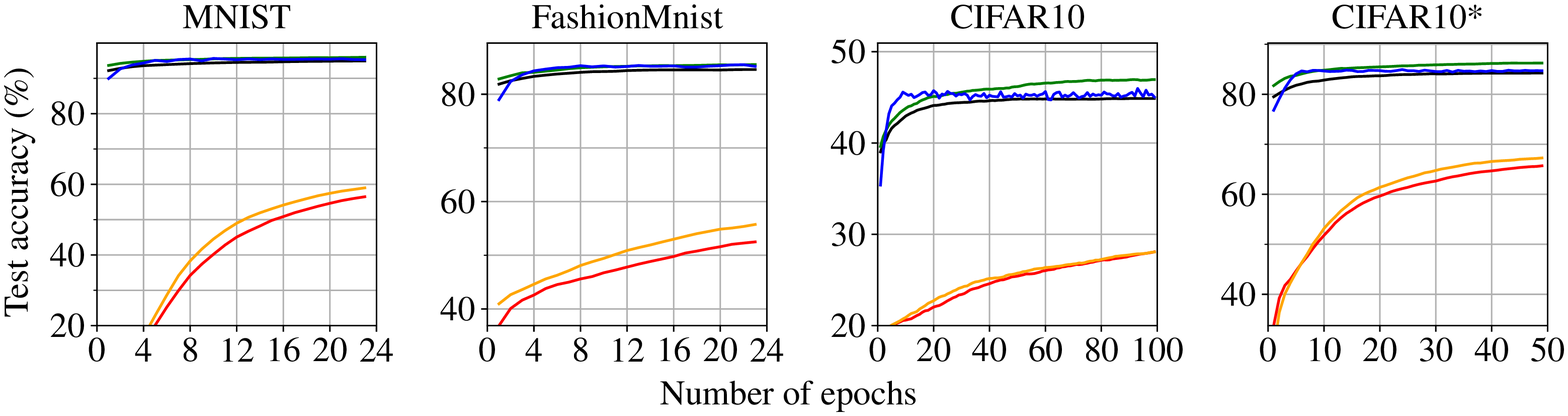}%
        \end{subfigure}
        
        \begin{subfigure}{1\textwidth}
        \centering 
        \hspace{0.9cm}
        \includegraphics[width=0.9\linewidth]{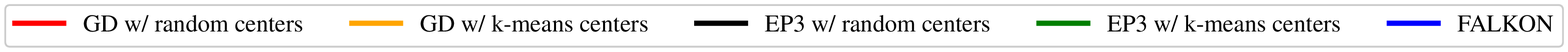}%
        \end{subfigure}
        
        \caption{\textbf{(Large scale training.)} 
      This figure shows the slow convergence of gradient descent given in \eqref{eq:gradient_descent} compared to our algorithm and \FALKON~from \citet{rudi2017falkon}.
      Note that \FALKON~involves a matrix inverse for a projection operation and hence converges faster with respect to the number of epochs.
}

    \label{fig:baseline}
\end{figure*}
If you plug in the form of general kernel models into \eqref{eq:kernel_regression}, we get
\begin{align}
    \minimize{\alphavec} L(\alphavec)&=\sum_{i=1}^n L(\sum_{j=1}^p K(\x_i,\z_j)\alpha_j,y_i) + \lambda \inner{\sum_{j=1}^p K(\cdot,\z_j),\sum_{j=1}^p K(\cdot,\z_j)}_\Hilbert\\
    &=\sum_{i=1}^n L( K(X_i,Z)\alphavec,y_i)+\lambda\alphavec\tran K(Z,Z)\alphavec.
\end{align}
For the square loss this is equivalent to
\begin{align}
    \minimize{\alphavec}\norm{K(X,Z)\alphavec-\y}^2+\lambda\alphavec\tran K(Z,Z)\alphavec.
\end{align}
Gradient descent on this problem for the square loss yields the update equation,
\begin{align}\label{eq:gradient_descent}
    \alphavec^{t+1}=\alphavec^t -\eta K(Z,X)((K(X,Z)\alphavec^t-\y)-\eta\lambda K(Z,Z)\alphavec.
\end{align}

\end{document}